\documentclass[letterpaper, 10 pt, conference]{ieeeconf}  

\IEEEoverridecommandlockouts                              

\usepackage{cite}
\usepackage{multicol}
\usepackage{algorithm}
\usepackage[noend]{algpseudocode}

\usepackage{amssymb}
\usepackage{amsmath}
\usepackage{centernot}
 
\DeclareMathOperator*{\argmax}{argmax}

\usepackage{amsthm}
\newtheorem{definition}{Definition}
\newtheorem{theorem}{Theorem}

\newtheorem{lemma}{Lemma}
\newtheorem{remark}{Remark}

\newtheorem{proposition}{Proposition}

\usepackage{graphicx}
\usepackage{svg}
\usepackage{float}

\usepackage{multirow}

\makeatletter
\let\NAT@parse\undefined
\makeatother
\usepackage{url}
\usepackage[bookmarks=true]{hyperref}

\newcommand{\hquad}{\hspace{0.5em}}

\newtheorem{assumption}{Assumption}

\begin{document}

\title{\LARGE \bf Robust STL Control Synthesis under Maximal Disturbance Sets}


\author{Joris Verhagen$^{1}$, Lars Lindemann$^{2}$ and Jana Tumova$^{1}$
\thanks{This work was partially supported by the Wallenberg AI, Autonomous Systems and Software Program (WASP) funded by the Knut and Alice Wallenberg Foundation.}
\thanks{$^{1}$ Joris Verhagen and Jana Tumova are with the Division of Robotics, Perception and Learning, School of Electrical Engineering and Computer Science, KTH Royal Institute of Technology, Stockholm, Sweden \{\tt\small jorisv, tumova\} @kth.se }
\thanks{$^{2}$Lars Lindemann is with the Thomas Lord Department of Computer Science, University of Southern California, Los Angeles CA, USA \tt\small llindema@usc.edu }
}

\maketitle
\thispagestyle{empty}
\pagestyle{empty}

\begin{abstract}
    This work addresses maximally robust control synthesis under unknown disturbances. 
    We consider a general nonlinear system, subject to a Signal Temporal Logic (STL) specification, and wish to jointly synthesize the maximal possible disturbance bounds and the corresponding controllers that ensure the STL specification is satisfied under these bounds. 
    Many works have considered STL satisfaction under given bounded disturbances. 
    Yet, to the authors' best knowledge, this is the first work that aims to maximize the permissible disturbance set and find the corresponding controllers that ensure satisfying the STL specification with maximum disturbance robustness. 
    We extend the notion of disturbance-robust semantics for STL, which is a property of a specification, dynamical system, and controller, and provide an algorithm to get the maximal disturbance robust controllers satisfying an STL specification using Hamilton-Jacobi reachability. 
    We show its soundness and provide a simulation example with an Autonomous Underwater Vehicle (AUV).
\end{abstract}


\section{Introduction}
In planning and control with spatial-temporal specifications, such as with Signal Temporal Logic (STL), maximizing robustness is often associated with certain characteristics of the system trajectory.
This is particularly the case for both STL space robustness and time robustness, and any variations of these. For example, space robustness~\cite{donze2010robust,dhonthi2021study} is, in simple terms, defined as to what extent the trajectory can be displaced (e.g. shifted in a 2D plane for planar problems) while still satisfying the specification. If the trajectory can be displaced by e.g., 1 meter, then the space robustness is $\rho = 1$. 
Similarly, time robustness~\cite{donze2010robust,lindemann2022temporal} is, in simple terms, defined as to what extent the trajectory can be shifted in time (e.g. each x and y trajectory gets shifted to the left or right on the time axis) while still satisfying the specification. If the robot can delay or advance the execution of a motion plan by e.g. 1 second, the time robustness is $\theta = 1$.

The limitation of considering these robustness metrics is that they are not related to the dynamics and controllability of the system.
We instead propose synthesizing controllers that are {maximally robust to external disturbances} while satisfying a spatial-temporal specification. \emph{Disturbance robustness}, as we will show, is a property of a specification, a dynamical system, and a controller. 
\\\\
For Autonomous Underwater Vehicles (AUVs), for example, bounding external disturbances like underwater currents is challenging. Overestimating these disturbances can lead to infeasibility or misjudged risks. Our goal is to develop controllers that accommodate the maximum possible disturbances while meeting the STL specification.
Consider Fig.~\ref{fig:intro} where an AUV is required to stay in $A$ during the interval $[7,8]$. The top scenario has a higher space robustness (distance within $A$) and time robustness (faster arrival and departure of $A$). The bottom scenario has a higher disturbance robustness due to the reduced surface area exposed to current pushing the AUV towards the rock. 
\\\\
The novel problem formulation requires simultaneously determining the maximal disturbance set and the controllers that achieve the STL specification. 
We use Hamilton-Jacobi reachability~\cite{bansal2017hamilton}, specifically concatenations of Backwards Reachable Sets (BRSs), to ensure the system can be controlled to satisfy the STL specification under worst-case disturbances from the largest permissible disturbance set.

\begin{figure}[t!]
    \centering
    \includegraphics[width=0.4\textwidth]{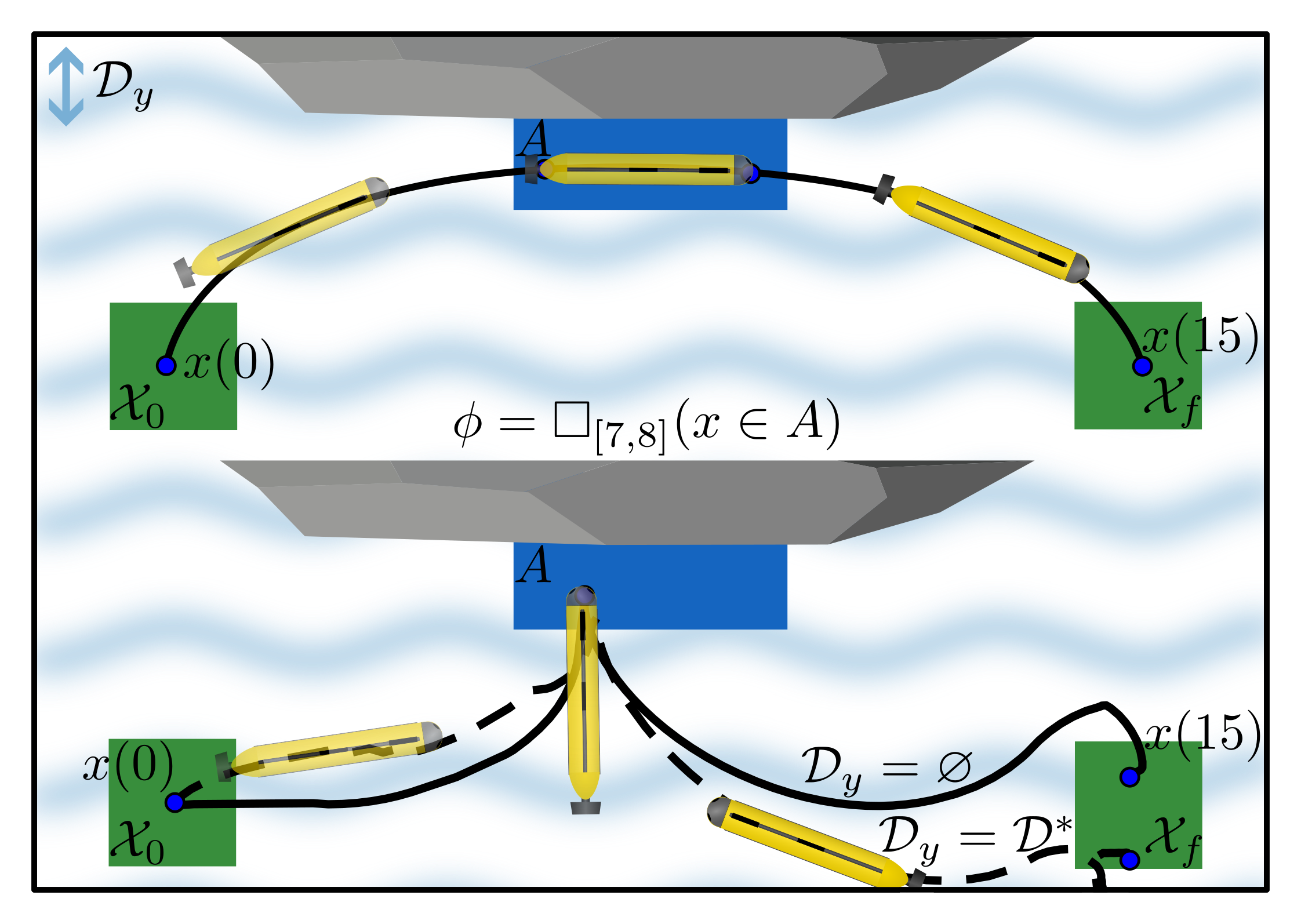}
    \caption{An illustrative reference trajectory (top) and the simulated trajectories from our method (bottom) of an AUV required, to be in $A$ between $7$ and $8$ seconds. The top scenario is more space- and time-robust while the bottom scenario is more disturbance-robust. Solid lines indicate the undisturbed trajectory (using the maximal disturbance robust controllers), the dashed line with the obtained maximal disturbance set.}
    \label{fig:intro}
    \vspace{-0.5cm}
\end{figure}


\subsection{Contributions}
\label{ssec:contributions}
In this work, we consider a time-independent nonlinear system subjected to a fragment of bounded-time STL. 
\begin{itemize}
    \item We introduce computationally efficient underaproximations of disturbance robustness for a fragment of STL.
    \item We present an algorithm to obtain maximally disturbance-robust controllers for this STL fragment.
    \item We prove its soundness and show empirical evidence of the effectiveness of our method in simulation.
\end{itemize}

\subsection{Related Work}
\label{ssec:related_work}
Robust planning and control synthesis under known disturbances has garnered substantial interest. 
Notably,~\cite{sadraddini2015robust,farahani2015robust} employ Model Predictive Control to guarantee maximally satisfying satisfaction of an STL formula in terms of space robustness semantics. Conversely,~\cite{raman2015reactive} presents counterexample-guided control synthesis subjected to bounded disturbances which can also be used in a receding horizon control method. The work in~\cite{majumdar2017funnel} employs Sum Of Squares programming to create a library of robust motion primitives. Integrating this with planning under STL specifications is explored in, for example, ~\cite{barbosa2019integrated, tajvar2020safe}.
In contrast, the works of~\cite{fan2020fast,chen2021scalable} present an innovative method utilizing translational and rotational invariance to construct reachability envelopes for safe motion planning, simplifying the motion planning process into a piecewise linear problem that accounts for dynamics and control, with extensions to space- and time-robust STL in~\cite{sun2022multi,verhagen2023temporally} respectively.

Employing reachability analysis to meet spatio-temporal criteria,~\cite{chen2018signal} establishes the link between reachability analysis and STL, decomposing complex STL formulas into sequences of reachability problems. 
In~\cite{yu2021online} online control synthesis under STL specifications is considered by constructing timed trees of the STL specification and transforming the problem into a tree of tubes. 
Our approach sacrifices some STL expressiveness but, instead of working on a given bounded disturbance set, aims to find the maximum permissible disturbance set while still satisfying the STL specification.

In~\cite{zhang2023investigating}, the problem of finding the largest possible disturbance set is addressed, posed as a controller falsification problem given the dynamical system, the controller, and the STL specification. 
The work presented here, to the best of the authors' knowledge, is the first to focus on synthesizing controllers that achieve maximal disturbance robustness while fulfilling STL specifications and employing disturbance robustness as a metric in synthesis.

\section{Preliminaries}
\label{sec:preliminaries}

Let $\mathbb{R}$ and $\mathbb{N}$ be the set of real and natural numbers including zero, respectively, and let $\mathbb{R}_{\geq 0}$ denote the set of real, non-negative numbers. Let $x(t) \in \mathbb{R}^n$ be the $n$-dimensional state of a system at time $t\in \mathbb{R}_{\geq 0}$. Let $\bar{A}$ and $\underline{A}$ denote the upper- and lower-bound respectively of the closed interval $A \subset \mathbb{R}_{\leq0}$ and let $\Delta A = \bar{A} - \underline{A}$ be its range. 


\subsection{System Dynamics and Trajectories}
\label{ssec:system_dynamics}
Consider the following, nonlinear dynamical system
\begin{equation}
\label{eq:dynamics}
    \dot{x} = f(x,u,d),
\end{equation}
where $x \in \mathcal X \subset \mathbb{R}^n$, $u \in \mathcal U \subset \mathbb{R}^m$, and $d \in \mathcal D \subset \mathbb{R}^p$ indicate time, state, control input, and bounded disturbances from the state, input and disturbance sets $\mathcal X, \mathcal U, \mathcal D$, respectively. The set $\mathcal D$ is bounded, but the exact bounds are unknown. Assume we are given a controller of the form $K: \mathbb{R}_{\geq 0} \times \mathcal X \rightarrow \mathcal U$, which gives rise to a closed-loop system, ${\dot{x} = f(x,K(t,x),d)}$.
The system's evolution can be formulated as an interaction between the controller and the disturbance.
We can then define 
\begin{equation}
\label{eq:funnel}
\begin{aligned}
\mathcal{F}(t) & =  \Phi_f(t; \mathcal{X}_0, K, \mathcal{D}) \\ & = \{ x(t) \mid  x(0) \in \mathcal{X}_0,  \dot{x} = f(x, K(t,x), d), \forall d \in \mathcal{D} \},
\end{aligned}
\end{equation}
which states that under controller $K(t,x)$, any disturbance $d \in \mathcal{D}$, and $\mathcal{X}_0 \subseteq \mathbb{R}^n$ as initial state set, the evolution of the system will remain in the time-dependent set $\mathcal{F}(t)$ at time $t$.

\subsection{Hamilton-Jacobi-Isaacs Reachability}
\label{ssec:HJ_reachability}

A choice of a controller $K$ in Eq.~\eqref{eq:funnel} can ensure specific reach-avoid behavior within certain time bounds regardless of the disturbances. First, let us define the Backwards Reachable Set.

\begin{definition}
Given a target set $\mathcal X_f$, the set of states from which the target can be reached at time $t$ is the Backwards Reachable Set (BRS) defined as 
\begin{equation}
\label{eq:BRS}
\begin{aligned}
    \mathcal{G}(t;t_f, \mathcal{X}_f, \mathcal D) = \{x(t) \mid  \exists u \in \mathcal{U}, \forall d(\cdot) \in \mathcal{D}, \\ \Phi_f(t_f;x,u,d) \in \mathcal{X}_f\},
\end{aligned}
\end{equation}

\label{def:BRS}
\end{definition}



We pose the challenge of finding a controller $K$ that ensures evolution to $\mathcal{X}_f$ under the worst-case disturbance from $\mathcal{D}$ as a zero-sum game, with a \emph{controller player} and a \emph{disturbance player}.
The BRS in Eq.~\eqref{eq:BRS} can be obtained by transforming the problem from a ``game of kind" to a ``game of degree" using the \emph{level set method}. We obtain the BRS by solving the Hamilton-Jacobi-Isaacs (HJI) PDE~\cite{evans1984differential}
\begin{equation}
\label{eq:HJI}
\begin{aligned}
    \frac{\partial}{\partial t}V(t,x) + &\max_{u \in \mathcal{U}} \min_{d \in \mathcal{D}}\nabla V(t,x) f(x,u,d) = 0, \\ &V(0,x) = v(x)
\end{aligned}
\end{equation}
which indicates the steady-state solution of the value iteration backward in time of the initial value assigned to the target with the scalar function $v(x) : \mathbb{R}^n \rightarrow \mathbb{R}$ that maps the state to a positive value if $x \in \mathcal{X}_f$ and negative otherwise. For any $x(t) \in \mathcal{G}(t;t_f,\mathcal{X}_f,\mathcal{D})$, the optimal controller is given by
\begin{equation}
\label{eq:opt_control}
    K^*(t,x) = \argmax_{u\in \mathcal{U}}\min_{d\in \mathcal{D}} \nabla V(t,x) f(x,u,d),
\end{equation}
ensuring the objective at time $t_f$ under any disturbance strategy $d \in \mathcal{D}$.

We implicitly assume that the \emph{disturbance player} plays a non-anticipative strategy and has an instantaneous informational advantage, meaning the \emph{disturbance} player can only react to already executed control decisions and that it makes its decision after the control player.
More information on two-player zero-sum differential games and level set methods can be found in~\cite{bansal2017hamilton}.

\subsection{Signal Temporal Logic}
\label{ssec:STL}
Signal Temporal Logic (STL)~\cite{maler2004monitoring} considers real-valued signals making it a powerful tool for specifying and verifying properties of dynamical systems.
Let us first define a fragment of STL over $n$-dimensional, finite, continuous-time signals $\mathbf{x}:\mathbb{R}_{\ge 0} \rightarrow X \subseteq \mathbb{R}^n$.

\begin{definition}[Fragment of Signal Temporal Logic]
\label{def:stl}
    Let bounded time intervals $I$ be in the form $[t_1,t_2]$, where, for all, $I \subset \mathbb{R}_{\ge 0}$, $t_1,t_2 \in \mathbb{R}_{\ge 0}, t_1 \leq t_2$.
    Let $\mu:X \rightarrow \mathbb{R}$ be a real-valued function, and let $p:X \rightarrow \mathbb{B}$ be a predicate defined according to the relation $p(\mathbf{x}) := \mu(\mathbf{x}) \ge 0$. The set of predicates is denoted $\textit{AP}$.
    We consider a fragment of STL, recursively defined as
    \begin{equation*}
    \begin{aligned}
        \psi &::= p \mid \diamondsuit_{I}p \mid \Box_{I} p \mid p_1 \mathcal{U}_{I} p_2 \\
        \phi &::= \psi \mid \phi_1 \land \phi_2 \mid \phi_1 \lor \phi_2 
    \end{aligned}
    \end{equation*}
    The fragment hence allows the bounded {\em Always} ($\Box_I$), {\em Eventually} ($\diamondsuit_I$), and {\em Until} ($\mathcal{U}_I$) operators and allows conjunctions and disjunctions of any combinations of these.
\end{definition}
We refer the reader to~\cite{maler2004monitoring} for the definition of these semantics.
The \emph{Until} operator, $\mathcal{U}_I$, specifies that $p_1$ should hold until, within $I$, $p_2$ holds. The {\em Always} and {\em Eventually} operators require to have a predicate hold for all time $t \in I$ and for some time $t \in I$, respectively. 
We use $\mathcal{X}_p = \{x \mid p(x) = \top \}$ to denote all states that satisfy predicate $p$. 
As $\mu: \mathbb{R}^n \rightarrow \mathbb{R}$ maps vectors in $\mathbb{R}^n$ to real numbers, we extend the definition of $\mu$ to operate on a set of vectors $\mathcal{X} \subseteq \mathbb{R}^n$ as $\mu(\mathcal{X}) = \min\{\mu(x) \mid x \in \mathcal{X}\}$ for notation convenience. 

Additionally, we say that temporal operator $\diamondsuit_{I}p$ has temporal flexibility since $p$ has to be met at an arbitrary $t' \in I$. 
Similarly, $p_1\mathcal{U}_I p_2$ has temporal flexibility since $p_2$ has to be met at an arbitrary $t' \in I$. Resolving temporal flexibility means determining these $t'$ values. 




As $\phi$ is a Boolean combination of STL subformulae, it can be transformed into a logically equivalent formula in the disjunctive normal form: ${\phi = \bigvee_{i=1}^M (\bigwedge_{j=1}^{N_i} \psi_{ij})}$.

\medskip

The real-valued predicate function $\mu$ allows a quantitative measure of satisfaction of the STL specification. The space-robustness of an STL specification $\phi$ on a reference trajectory $x$ at time $t$ can be recursively computed according to~\cite{maler2004monitoring}
\begin{equation}
\begin{aligned}
\label{eq:space_robustness}
    \rho_p(t,x) &= \mu(x), \\
    \rho_{\Box_I p}(t,x) &= \min_{t'\in I} \rho_{p}(t',x), \\
    \rho_{\diamondsuit_I p}(t,x) &= \max_{t'\in I} \rho_{p}(t',x), \\
    \rho_{p_1 \mathcal{U}_I p_2}(t,x) &= \max_{t_1 \in I+t} \min \{\rho_{p_2}(t,x),\max_{t_2 \in [t,t_1]}\rho_{p_1}(t,x)\}, \\
    \rho_{\phi \land \psi}(t,x) &= \min\{\rho_{\phi}(t,x),\rho_{\psi}(t,x)\}, \\
    \rho_{\phi \lor \psi}(t,x) &= \max\{\rho_{\phi}(t,x),\rho_{\psi}(t,x)\}.
\end{aligned}
\end{equation}
\\\\
In~\cite{zhang2023investigating}, the notion of disturbance robustness is introduced for the falsification of a given controller. In contrast to space- and time-robustness of a specification $\phi$, this notion considers robustness of the closed-loop system
\begin{equation}
\begin{aligned}
\label{eq:dist_robust_spec}
    \delta_{\phi} = \max \quad&|\mathcal{D}|, \\
    \textrm{s.t.} \quad &\rho_{\phi}(t,\Phi_f(t;\mathcal{X}_0,K,\mathcal{D})) \geq 0, 
\end{aligned}
\end{equation}
which considers finding the maximum set $\mathcal{D}$ (defined later), under a given controller $K(t,x)$, such that any closed-loop evolution under the controller and any disturbance strategy has non-negative space robustness.
\section{Problem Statement}
\label{sec:problem_statement}

First, we define the disturbance-robust satisfaction of an STL specification.

\begin{definition}[Disturbance-Robust STL]
\label{def:dist_robust_STL}
    A control strategy $K(t,x)$ is disturbance robust w.r.t. the disturbance set $\mathcal{D}$ and specification $\phi$ if 
    \begin{equation}
        \Phi_f(t;\mathcal{X}_0,K,\mathcal{D}) \models \phi, \quad \forall t \in [t_0,t_f]
    \end{equation}
    where $t_0$ and $t_f$ are the start- and end time of $\phi$.
\end{definition}

This requires that all solutions of the closed-loop system satisfy the STL specifications under controller $K(t,x)$ against all possible non-anticipative strategies from $d \in \mathcal{D}$.

We can now state the optimization problem that we address in this work; maximal disturbance-robust Signal Temporal Logic control synthesis
\begin{align}
    \delta_{\phi} = \max_{K,\mathcal{D}} \quad &|\mathcal{D}|, \label{eq:prob}\\
    \textrm{s.t.} \quad & \Phi_f(t;\mathcal{X}_0,K,\mathcal{D}) \models \phi, \quad \forall t \in [t_0,t_f]. \tag{9a}\label{eq:prob_c1}
\end{align}

Notice that we jointly optimize for the time-dependent state-feedback controller $K(t,x)$ and the maximum disturbance set. 
This is a challenging task due to the interplay between controller and disturbances and the history-dependency of closed-loop system behavior. As such, we approach this problem via concatenations of reachability problems, which differs from methods for space- and time-robustness.

Motivated by the application to real-world robotics, we define the disturbance-robustness degree $\kappa = |\mathcal{D}|$ according~to
\begin{align*}
    \kappa = \max \quad&d_{\infty}, \\
    \textrm{s.t.} \quad&\{y\in \mathbb{R}^k \big| \|y\|_{\infty} \leq d_{\infty}\} \subseteq \mathcal{D}. 
\end{align*} 
In other scenarios, one might wish to displace this origin and scale the dimensions of the disturbance set differently.

\begin{remark}
    The disturbance-robust metric in Eq.~\eqref{eq:prob} is implicitly w.r.t. unknown disturbance bounds. One can additionally consider positive and negative disturbance-robustness by considering the distance to a known, approximated, or learned disturbance set, $\bar\delta_{\phi} = \mathit{dist}(\mathcal{D},\mathcal{D}_{\mathit{known}})$
\end{remark}

\section{Disturbance-Robust Semantics}
\label{sec:disturbance_robust_semantics}

In contrast to existing notions of space- and time-robustness, it follows from the dependency on closed-loop behavior in Eq.~\eqref{eq:prob_c1} that disturbance-robustness of a specification is history-dependent and cannot be obtained by evaluation at just the intervals of the temporal operators as is done for space-robustness in Eq.~\eqref{eq:space_robustness}.
As such, we first wish to expand upon the existing notion of specification disturbance robustness in Eq.~\eqref{eq:dist_robust_spec}.
\\\\
Consider an STL specification in the fragment defined in Def.~\ref{def:stl}, given in its disjunctive normal form $\phi = \bigvee_{i=1}^M (\bigwedge_{j=1}^{N_i} \psi_{ij}) = \bigvee_{i=1}^N \Psi_i$. Each subformula $\Psi_i$ represents a possible way to achieve $\phi$. 
To address problem~\eqref{eq:prob}, we state the following additional assumptions.
\begin{assumption}[Bounded Predicates]
\label{ass:predicates}
    Predicates $p(x)$ in specification $\phi$ are closed bounded sets of the state $x$ or conjunctions and disjunctions thereof
    \begin{equation*}
        p(x) = \top \implies x \in (\bigcap_{i} S_i) \cup (\bigcup_{j} S_j)
    \end{equation*}
    for bounded sets $S_i,S_j \subset \mathbb{R}^n$.
\end{assumption}

Furthermore, for the sake of presentation simplicity, we state the following assumption. 
\begin{assumption}[Non-overlapping Intervals]
\label{ass:intervals}
    All $N_i$ operators within subformula $\Psi_i$ have non-overlapping intervals $I$
    \begin{equation*}
        I_{i,j} \cap (\bigcup_{j\neq k} I_{i,k}) = \varnothing, \quad \forall j,k \in \{1,...,N_i\}.
    \end{equation*}
\end{assumption}

Based on the fragment in Def.~\ref{def:stl} and these assumptions, we can state some observations regarding operator-based disturbance-robust semantics. First, consider the definition of the disturbance robustness of a predicate $p$ at any time $t$

\begin{equation}
\label{eq:delta_p}
\begin{aligned}
    &\hat\delta_{p}(K,\mathcal{X}_0,t_0,t) := \\
    &\argmax_{|\mathcal{D}|} \mu(\Phi_f(t-t_0,\mathcal{X}_0,K,\mathcal{D})) \geq 0, 
\end{aligned}
\end{equation}
w.r.t the set of initial states $\mathcal{X}_0$, initial time $t_0$, and dynamics and controller $f$ and $K$. In contrast to the definition in Eq.~\eqref{eq:dist_robust_spec} from~\cite{zhang2023investigating}, this definition allows us to consider the disturbance-robustness of a single predicate as opposed to the whole specification $\phi$. 
For the temporal operators, we make the following observations

\begin{equation}
\label{eq:delta}
\begin{aligned}
    \hat\delta_{\Box_I p}(K,\mathcal{X}_0,t_0) &:= \min\{\hat\delta_p(K,\mathcal{X}_0,t_0,\underline{I}), \\ &\quad\quad \min_{t'\in I}\hat\delta_p(K,\mathcal{X}_p,\underline{I},t')\}, \\
    \hat\delta_{\diamondsuit_I p}(K,\mathcal{X}_0,t_0) &:= \max_{t' \in I}(\hat\delta_{p}(K,\mathcal{X}_0,t_0,t')), \\
    \hat\delta_{p_1 \mathcal{U}_I p_2}(K,\mathcal{X}_0,t_0) &:= \min\{ \hat\delta_{p_1}(K,\mathcal{X}_0,t_0,\underline{I}), \\ \max_{t' \in I} (\min&(\hat\delta_p(K,\mathcal{X}_{p_1},\underline{I},t'), \hat\delta_{p_2}(K,\mathcal{X}_{p_1},t',t'))) \} \\
    \hat\delta_{\psi_1 \land \psi_2}(K,\mathcal{X}_0,t_0) &:= \min(\hat\delta_{\psi_1}(K,\mathcal{X}_0,t_0), \hat\delta_{\psi_2}(K,\mathcal{X}_{p_1},t_1')), \\
    \hat\delta_{\psi_1 \lor \psi_2}(K,\mathcal{X}_0,t_0) &:= \max(\hat\delta_{\psi_1}(K,\mathcal{X}_0,t_0),\hat\delta_{\psi_2}(K,\mathcal{X}_0,t_0)).
\end{aligned}
\end{equation}
Here, $\hat\delta_{\Box_I p}$ explicitly considers the minimum disturbance-level to go from $\mathcal{X}_0$ to $\mathcal{X}_p$ (via $\hat\delta_p(K,\mathcal{X}_0,t_0,\underline{I})$) and to stay in $\mathcal{X}_p$ (via $\min_{t'\in I}\hat\delta_p (K,\mathcal{X}_p,\underline{I},t')$).
Note that $\hat\delta_{\psi_1 \land \psi_2}$ explicitly considers the going from $\mathcal{X}_0$ to $\mathcal{X}_{p_1}$ and $\mathcal{X}_{p_1}$ to $\mathcal{X}_{p_2}$.
These observations explicitly consider the sequence of temporal operators due to the history dependence of disturbance robustness.

\begin{lemma}[Disturbance-Robust Recursive Semantics]
    A specification $\phi$ from Def.~\ref{def:stl} with Assumptions~\ref{ass:predicates} and~\ref{ass:intervals}. The disturbance-robust recursive semantics on this fragment from Eq.~\eqref{eq:delta_p} and~\eqref{eq:delta} constitute an underapproximation of the maximal disturbance set $\mathcal{D}$ for which $\Phi_f(t;x_0,K,\mathcal{D}) \models \phi, \forall x_0 \in \mathcal{X}_0$, e.g. $\hat\delta_{\phi} \leq \delta_{\phi}$. 
\end{lemma}
\begin{proof}[Proof Sketch]
    Consider $\phi = \Box_{I_1}p_1 \land \Box_{I_2}p_2$ where $\bar{I}_1 \leq \underline{I}_2$ and we have obtained $\mathcal{D}^*$. 
    $\hat\delta_{\Box_1 \land \Box_2}$ requires that $x(0) \in \mathcal{X}_0 \implies x(t) \in \mathcal{X}_{p_1} \forall t \in I_{1}, \forall d \in \mathcal{D}$ and $x(\bar{I}_1)\in \mathcal{X}_{p_1} \implies x(t) \in \mathcal{X}_{p_2} \forall t \in I_{2}, \forall d \in \mathcal{D}$. 
    Consider however that there might exists a $K(t,x)$ and $(\mathcal{D}_1,\mathcal{D}_2)$ for which $\mathcal{D}_1 \cap \mathcal{D}_2 \supset \mathcal{D}^*$ such that $x(0) \in \mathcal{X}_0 \implies x(t) \in \mathcal{X}_{p_1} \forall t \in I_1 \land x(t) \in \mathcal{X}_{p_2} \forall t \in I_2$ where $\exists x(\bar{I}_1) \in \mathcal{X}_{p_1}$ for which $x(\underline{I}_2) \not\in \mathcal{X}_{p_2}$.
    In short, robustness is defined over entire sets. There might exists less conservative trajectories within those sets.
\end{proof}


The concatenation of disturbance robust temporal operators via $\hat\delta_{\phi_1 \land \phi_2}$ motivates us to similarly consider the concatenation of backward reachable sets to solve Eq.~\eqref{eq:prob}.
\section{Disturbance-robust Synthesis}
\label{sec:disturbance_robust_synthesis}
Finding the true maximally disturbance-robust hybrid controller for a dynamical system subject to an STL specification is intractable due to the history-dependent nature of reachable sets. As such, we aim to find an approximate solution in a discrete search space where we asses sequentially decreasing disturbance sets. The fragment of STL in Def.~\ref{def:stl}, the assumptions~\ref{ass:predicates} and~\ref{ass:intervals}, and the newly established disturbance-robust semantics in Eq.~\eqref{eq:delta} allow us to do this efficiently.
We also make use of the following proposition.


\begin{figure}[t!]
    \centering
    \includegraphics[width=0.45\textwidth]{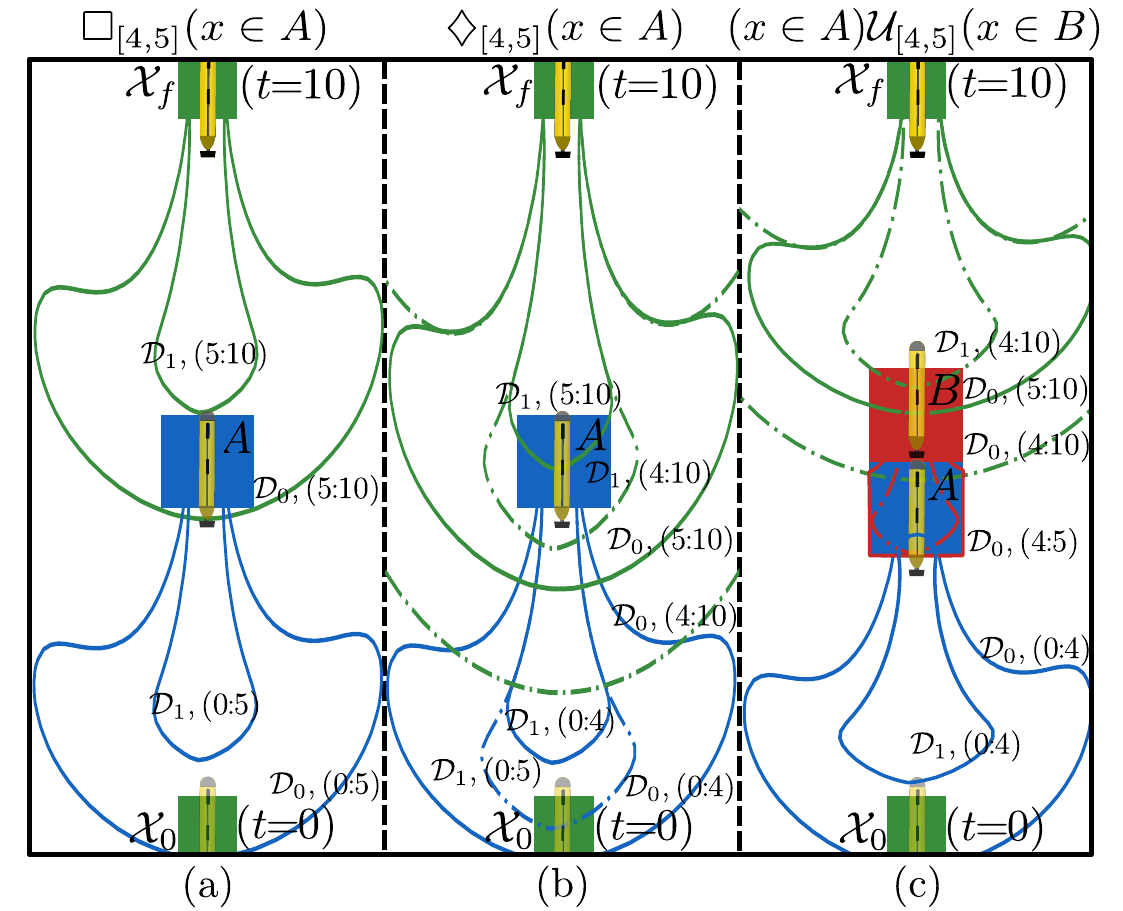}
    \caption{Concatenation of reachable sets. $(t_1,t_2)$ indicates between which times the optimal controller corresponding to the BRS will be active. a) for \emph{Always}, the reachable sets should intersect at $t_5$ and contain a stay-in reachability computation from $t_5$ to $t_4$ (blue). The only valid path is $(\mathcal{D}_0,0\mathrel{\!\!:\!\!}5),(\mathcal{D}_0,5\mathrel{\!\!:\!\!}10)$, b) for \emph{Eventually}, $A$ can be met at $t_4$ or $t_5$ (solid and dashed lines). The optimal valid path is $(\mathcal{D}_0,0\mathrel{\!\!:\!\!}4),(\mathcal{D}_1,4\mathrel{\!\!:\!\!}10)$, $\mathcal{D}^* = \mathcal{D}_0$, c) for \emph{Until}, we consider three set concatenations. Blue sets need strict satisfaction at $\underline{I}$, red ensures $A\rightarrow B$ for any $t\in I$ and green ensures $B\rightarrow \mathcal{X}_f$ for that same $t$.}
    \label{fig:algorithm}
    \vspace{-0.5cm}
\end{figure}

\begin{proposition}
\label{lem:shifting}
    The time-invariant dynamics and predicates in Eq.~\eqref{eq:dynamics} and Def.~\ref{def:stl} respectively, permit time-shifting Backward Reachable Sets (BRSs). Therefore, a BRS ensuring predicate satisfaction at $t=t'$ starting at $t=t_0, t'\geq t_0$ can equivalently ensure satisfaction at $t=t''$ starting at $t=t_0+(t''-t'), \forall t_0,t',t'' \in \mathbb{R}_{\geq 0}$.
\end{proposition}

Additionally, note that $\psi_{i,j+1}$ precedes $\psi_{i,j}$ as we conduct BRS computations from the last to the first operator in $\Psi_i$. Let $I_{i,j}$ and $p_{i,j}$ be the interval and predicate of $\psi_{i,j}$. Lastly, remember we have $M$ subformulae and subformula $\Psi_i$ has $N_i$ temporal operators.

\subsection{Disturbance-Robustness on Specification Level}
In the following, we look for the subformula with the maximal disturbance robustness and obtain the corresponding hybrid controller. This is equal to finding the maximal disturbance robustness for $\phi$ due to its disjunctive normal form.

\begin{algorithm}
\caption{Maximal Disturbance-robust Control Synthesis}\label{alg:cap}
\begin{algorithmic}[1]
\Require $\dot{x}=f(x,u,d)$, $\phi$, $\mathcal{D}_0$
\State Order subformulae $\Psi_i$ backward in time
\State $\mathcal{D}^* \gets \varnothing, K^* \gets \varnothing$
\State $\mathcal{D}^{\mathit{unresolved}} \gets \varnothing, K^{\mathit{unresolved}} \gets \varnothing$
\For{$i \in \{1,...,M\}$}
    \State $\mathcal{D}_{\Psi_i} \gets \mathcal{D}_0$
    \For{$j \in \{1,...,N_i-1\}$}

        \State $\mathcal{D}_{\psi_{i,j}}, K_{\psi_{i,j}} \gets$ Alg.~\ref{alg:always},~\ref{alg:eventually},~\ref{alg:until}$(\mathcal{D}_{\Psi_i},\psi_{i,j},\psi_{i,j+1})$

        \State \textbf{if} $\mathcal{D}_{\psi_{i,j}} = \varnothing$ \textbf{then} break
        
        \If{$\psi_{i,j+1}  = \Box$ or $j = N_i-1$}
            \State $\mathcal{D}_{i,j} \gets$ Eq.~\eqref{eq:resolve}$(\mathcal{D}^{\mathit{unresolved}})$
            \State append respective $K_{i,j}$'s to $K_{\Psi_i}$
            \State \textbf{if} $\mathcal{D}_{i,j} \subseteq \mathcal{D}_{\Psi_i}$ \textbf{then} $\mathcal{D}_{\Psi_i} \gets \mathcal{D}_{i,j}$
            \State $\mathcal{D}^{\mathit{unresolved}} \gets \varnothing$, $K^{\mathit{unresolved}} \gets \varnothing$
        \Else
            \State add $\mathcal{D}_{\psi_{i,j}}, K_{\psi_{i,j}}$ to $\mathcal{D}^{\mathit{unresolved}}, K^{\mathit{unresolved}}$
        \EndIf
        
    \EndFor
    \State $(\mathcal{D}^*,K^*) \gets (\mathcal{D}_{\Psi_i},K_{\Psi_i})$ if $\mathcal{D}_{\Psi_i} \supseteq \mathcal{D}^*$
\EndFor

\end{algorithmic}
\end{algorithm}

Alg.~\ref{alg:cap} indicates how we obtain the maximally disturbance-robust hybrid controller $K^*$ (to the accuracy of $\epsilon$) given an STL specification from Def.~\ref{def:stl}.
We require an overapproximation of the true maximum disturbance set $\mathcal{D}_0$ from which we start iterative HJI computations.
We also keep track of disturbance sets and controllers for operators for which the temporal flexibility has not yet been resolved (\emph{line 3}).
We consider each subformula $\Psi_i$ (\emph{line 4}), initialize a new initial disturbance-robustness (\emph{line 5}) and loop through its temporal operators, $\psi_{i,j}$, backward in time (\emph{line 6}). 
For each temporal operator an algorithm, later described in Alg.~\ref{alg:always},~\ref{alg:eventually}, or~\ref{alg:until} (for $\Box$, $\diamondsuit$, or $\mathcal{U}$ respectively), we obtain maximally disturbance-robust controllers and the maximum disturbance sets, $\mathcal{D}_{\psi_{i,j}} \subseteq \mathcal{D}_{\Psi_i} \subseteq \mathcal{D}_0$. 
These controllers and sets are obtained as a function of the current maximum permissible disturbance set found in $\Psi_i$, $\mathcal{D}_{\Psi_i}$ and the current and preceding operators $\psi_{i,j}$, $\psi_{i,j+1}$ (\emph{line 7}). 
If no valid disturbance set exists, we discard the subformula (\emph{line 8}).
We will see later in Alg.~\ref{alg:always},~\ref{alg:eventually}, and~\ref{alg:until} that if either $\psi_{i,j}$ or $\psi_{i,j+1}$ are operators with temporal flexibility, we return sets of controllers and disturbance sets that satisfy the predicates of the operators at some point in the intervals.

\subsubsection{Resolving Temporal Flexibility}
For a sequence of \emph{Eventually} and \emph{Until} operators we need to resolve temporal flexibility (e.g. determine at which $t' \in I$, $p$ should hold to satisfy $\diamondsuit_{I}p$). We can do this when we encounter an \emph{Always} operator (which requires satisfaction of the predicate at a strict time, $\bar{I}$) or we have assessed the last operator (\emph{line 9}).
We resolve temporal flexibility with a search problem that finds the maximum joint disturbance set $\mathcal{D}_{\Psi_i}$ while ensuring strict temporal concatenations of the BRSs within the operator's time intervals.
\begin{equation}
\begin{aligned}
\label{eq:resolve}
    \mathcal{D}_{\Psi_i} = &\min(\max(\mathcal{D}_{\psi_{i,k}}),...,\max(\mathcal{D}_{\psi_{i,k+}})), \\
    s.t. \quad &t_{i,j}^\mathit{departure} = t_{i,j+1}^\mathit{arrival}, \quad j=r,\ldots,s-1, \\
    &t_{i,j}^\mathit{departure} \in I_{i,j},
\end{aligned}
\end{equation}
where $\{r,...,s\}$ index the temporal operators that need to be resolved (from \emph{line 15}). $\mathcal{D}_{\psi_{i,j}} = \{\mathcal{D}^1,...,\mathcal{D}^m\}$ denotes all disturbance sets that satisfy $p_{i,j}$ for some $t \in I_{i,j}$. With some liberty in notation, $t_{i,j}^\mathit{departure}$ and $t_{i,j}^\mathit{arrival}$ denote, in absolute terms, the start and end time of BRS for $\psi_{i,j}$, capturing where in time the BRS will be placed and for how long the corresponding controller will be active. 
From Eq.~\eqref{eq:resolve} we obtain the maximum disturbance set for all resolved operators. The corresponding controllers are appended to the set of controllers for the subformula, $K_{\Psi_i}$ (\emph{line 11}) and the maximum disturbance set is lowered if required (\emph{line 12}). If $\Psi_i$ is a sequence of $\Box$ operators, we call Eq.~\eqref{eq:resolve} after every operator.
As $x(t) \in B \centernot\implies \exists u \in \mathcal{U}, s.t. \hquad \exists t' > t, x(t') \in B, \forall d \in \mathcal{D}$, arrival and departure times should strictly coincide. 
We keep track of the maximal disturbance set of $\phi$, $\mathcal{D}^*$ (\emph{line 16}).

\subsection{Disturbance-Robustness on Operator Level}
For notation convenience, we drop index $i$ and denote $\psi_{j}$ as a temporal operator in subformula $\Psi_i$. Subsequently, we use $p_{j}$ and $I_{j}$ for the predicate and interval.
For the concatenations of BRSs associated with these temporal operators, the initial state of the BRS must be completely contained in the final state of the preceding BRS; $\mathcal{G}_{j+1}(T) \supseteq \mathcal{G}_{j}(0)$ (as we sort $\Psi_i$ backward in time). 
If this is the case for disturbance sets $\mathcal{D}_{\psi_{j}}$ and $\mathcal{D}_{\psi_{j+1}}$, then it is guaranteed to hold $\forall \mathcal{D} \subseteq \mathcal{D}_{\psi_{j}} \cap \mathcal{D}_{\psi_{j+1}}$ as per Alg.~\ref{alg:cap}.

\subsubsection{Always Operator}
The procedure for the \emph{Always} operator is described in Alg.~\ref{alg:always}. We differentiate between the preceding operator $\psi_{j+1}$ (for which $\bar{I}_{j+1} < \underline{I}_{j}$) being an \emph{Eventually} or \emph{Until} operator (which has temporal flexibility) or an \emph{Always} operator (which does not). 
If the preceding operator has temporal flexibility, we conduct iterative HJI reachability computations (\emph{line 6}) with a decreasing disturbance set $\mathcal{D}$ (\emph{line 10}) until the BRS completely intersects the preceding predicate $\psi_{j+1}, \forall t \in I_{j+1}$, or the disturbance set is smaller than a feasible disturbance set which has been found for another subformula (\emph{line 5}). We keep track of all BRSs for which $\exists t \in I_{j+1}$ that completely intersects the preceding predicate $p_{j+1}$ (\emph{line 7 - 9}) for resolving with Eq.~\eqref{eq:resolve}. 
If the preceding operator has no temporal flexibility, we conduct the computations (\emph{line 11}), with decreasing $\mathcal{D}$, until the reachable set intersects the predicate of the preceding operator at $t = \bar{I}_{j+1}$ (\emph{line 12}).
To ensure the stay-in objective for the duration $\Delta I_{j}$, we add the obstacle $\neg p_{j}$ for the duration $\Delta I_{j}$ at the start of the HJI computation (\emph{line 6 and 14}). 
An intuitive example is shown in Fig.~\ref{fig:algorithm}a.

\begin{algorithm}[htb]
\caption{Disturbance-robust Always Operator}\label{alg:always}
\begin{algorithmic}[1]
\Require $\mathcal{D}_0$, $\Box_{I_{j}}p_{j}$, $\psi_{j+1}$
\State $\mathcal{D}_j \gets \mathcal{D}_0$
\State $\mathcal{D}_{\Box_j} \gets \varnothing$, $K_{\Box_j} \gets \varnothing$, $\mathcal{G}_{\Box_j} \gets \varnothing$

\If{$\psi_{j+1} \in \{\mathcal{U},\diamondsuit\}$}
    \While{$\exists t' \in I_{j+1}$ for which $\mathcal{G}_j(t') \not\supseteq \mathcal{X}_{p_{j+1}}$}

        \State \textbf{if} $\mathcal{D}_j \subseteq \mathcal{D}^*$ or $\mathcal{D}_j = \varnothing$ \textbf{then} return $\varnothing$
        
        \State $K_j,\mathcal{G}_j(t) \gets \mathcal{G}(\bar{I}_{j+1},p_{j},\mathcal{D}_j)$ with $\neg p_j \forall t \in I_{j}$

        \For{$t' \in I_{j+1}$ for which $\mathcal{G}_j(t') \supseteq \mathcal{X}_{p_{j+1}}$}
            \State append $\mathcal{D}_j$ to $\mathcal{D}_{\Box_j}$
            \State append $K_j(t':,x)$ to $K_{\Box_j}$, $\mathcal{G}_j(t':,x)$ to $\mathcal{G}_{\Box_j}$
        \EndFor
        \State $\mathcal{D}_j \gets \mathcal{D}_j - \epsilon$
    \EndWhile
\Else
    \While{$\mathcal{G}_j(\bar{I}_{j+1}) \not\supseteq \mathcal{X}_{p_{j+1}}$}

        \State \textbf{if} $\mathcal{D}_j \subseteq \mathcal{D}^*$ or $\mathcal{D}_j = \varnothing$ \textbf{then} return $\varnothing$
        
        \State $K_j,\mathcal{G}_j \gets \mathcal{G}(\bar{I}_{j+1},p_{j},\mathcal{D}_j)$ with $\neg p_j \forall t \in I_{j}$
        \State $\mathcal{D}_j \gets \mathcal{D}_j - \epsilon$
    \EndWhile
    \State append $\mathcal{D}_j$ to $\mathcal{D}_{\Box_j}$, $K_j$ to $K_{\Box_j}$, $\mathcal{G}_j$ to $\mathcal{G}_{\Box_j}$
\EndIf 

\State return $\mathcal{D}_{\Box_j}$, $K_{\Box_j}$, $\mathcal{G}_{\Box_j}$
\end{algorithmic}
\end{algorithm}

\subsubsection{Eventually Operator}
The \emph{Eventually} operator has temporal flexibility, regardless of the type of operator that proceeds. The procedure is described in Alg.~\ref{alg:eventually}. 
We convert the \emph{Eventually} operator to an \emph{Always} operator with interval $[\bar{I}_{j},\bar{I}_{j}]$.
If the preceding operator has temporal flexibility, the total temporal flexibility is over the interval $[\underline{I}_{j+1},\bar{I}_{j+1}+\Delta I_{j}]$ (\emph{line 3}), otherwise merely over $[\bar{I}_{j+1},\bar{I}_{j+1}+\Delta I_{j}]$ (\emph{line 5}) according to Prop.~\ref{lem:shifting}. An intuitive example of the temporal flexibility of the \emph{Eventually} operator is shown in Fig.~\ref{fig:algorithm}b.

\begin{algorithm}[!htb]
\caption{Disturbance-robust Eventually Operator}\label{alg:eventually}
\begin{algorithmic}[1]
\Require $\mathcal{D}_0$, $\diamondsuit_{I_{j}}p_{j}$, $\psi_{j+1}$
\State $\psi_{j} \gets \Box_{[\bar{I}_{j},\bar{I}_{j}],j}p_{j}$

\If{$\psi_{j+1} \in  \{\mathcal{U},\diamondsuit\}$}
    \State $\psi_{j+1} \gets \diamondsuit_{[\underline{I}_{j+1},\bar{I}_{j+1}+\Delta I_{j}]}p_{j+1}$
\Else
    \State $\psi_{j+1} \gets \diamondsuit_{[\bar{I}_{j+1},\bar{I}_{j+1}+\Delta I_{j}]}p_{j+1}$
    
\EndIf
\State $\mathcal{D}_{\diamondsuit_{j}}, K_{\diamondsuit_{j}}, \mathcal{G}_{\diamondsuit_{j}} \gets$ Alg.~\ref{alg:always}$(\mathcal{D}_0,\psi_{j},\psi_{j+1})$

\State return $\mathcal{D}_{\diamondsuit_{j}}$, $K_{\diamondsuit_{j}}$, $\mathcal{G}_{\diamondsuit_{j}}$
\end{algorithmic}
\end{algorithm}

\subsubsection{Until Operator}
For the \emph{Until} operator, there is temporal flexibility and dependency on two bounded predicates. The procedure is described in Alg.~\ref{alg:until}. We conduct two HJI reachability computations (\emph{line 4 and 8}): assessing backward reachability from $\mathcal{X}_{p_2}$ to $\mathcal{X}_{p_1}$ and from $\mathcal{X}_{p_1}$ to $\mathcal{X}_{p_{j+1}}$. Again, we return a collection of disturbance sets and BRSs. An example of BRSs for the \emph{Until} operator is shown in Fig.~\ref{fig:algorithm}c. The implementation of the first predicate satisfaction introduces conservatism as the $\Box_{[\underline{I}_{j},\max(t')]}$ operator (\emph{line 6}) requires stay-in behavior for a duration which might be longer than strictly necessary.

\begin{algorithm}[!h]
\caption{Disturbance-robust Until Operator}\label{alg:until}
\begin{algorithmic}[1]
\Require $\mathcal{D}_0$, $p_1\mathcal{U}_{I}p_2$, $\psi_{j+1}$
\State $\mathcal{D}_{2,j} \gets \mathcal{D}_0$
\State $\mathcal{D}_{\mathcal{U}_{j}} \gets \varnothing$, $K_{\mathcal{U}_{j}} \gets \varnothing$, $\mathcal{G}_{\mathcal{U}_{j}} \gets \varnothing$

\While{$\exists t' \in I_{j}$ for which $\mathcal{G}_2(t') \not\supseteq \mathcal{X}_{p_{1}}$}

    \State \textbf{if} $\mathcal{D}_{2,j} \subseteq \mathcal{D}^*$ or $\mathcal{D}_{2,j} = \varnothing$ \textbf{then} return $\varnothing$
    \State $K_2, \mathcal{G}_2 \gets \mathcal{G}(\underline{I}_j,p_2,\mathcal{D})$

    \If {$\exists t'\in I_j, \mathcal{G}_2(t') \supseteq \mathcal{X}_{p_{1}}$}
        \State $\psi_{j} \gets \Box_{[\underline{I}_j,\max(t')]}p_1$
        \While{$\exists t'' \in I_{j+1}$ for which $\mathcal{G}_1(t'') \not\supseteq \mathcal{X}_{p_{j+1}}$}
        \State $\mathcal{D}_1,K_1,\mathcal{G}_1(t) \gets$ Alg.~\ref{alg:always}$(\mathcal{D}_{2,j},\psi_{j}, \psi_{j+1})$

            \For {$(\mathcal{D}_{1,j},K_{1,j},\mathcal{G}_{1,j}) \in (\mathcal{D}_1,K_1,\mathcal{G}_1) $}
                \If {$\exists t'' \in I_{j+1}, \mathcal{G}^1_j(t'') \supseteq \mathcal{X}_{p_{j+1}}$}
                    \State append $(\mathcal{D}_{1,j} \cap \mathcal{D}_{2,j})$ to $\mathcal{D}_{\mathcal{U}_j}$, 
                    \State append $(K_{1,j},K_2)$ to $K_{\mathcal{U}_j}$, 
                    \State append $(\mathcal{G}_{1,j},\mathcal{G}_2)$ to $\mathcal{G}_{\mathcal{U}_j}$
                \EndIf
            \EndFor
        \EndWhile
    \EndIf
    \State $\mathcal{D}_{2,j} \gets \mathcal{D}_{2,j} - \epsilon$
\EndWhile

\State return $\mathcal{D}_{\mathcal{U}_j}$, $K_{\mathcal{U}_j}$, $\mathcal{G}_{\mathcal{U}_j}$
\end{algorithmic}
\end{algorithm}

Due to the temporal flexibility of $\diamondsuit_I$ and $\mathcal{U}_I$, we are interested in all $\mathcal{D}_i$ for which the respective BRSs that at some time intersect the predicate within interval $I$. This prevents us from using a golden section method to converge to the optimal solution. Eq.~\eqref{eq:resolve} resolves this temporal flexibility.

\begin{remark}
    In the context of Alg.~\ref{alg:always},~\ref{alg:eventually}, and~\ref{alg:until}, the evaluation of BRSs against predicates encompasses the entirety of the defined bounded predicate space (e.g. \emph{line 4} in Alg.~\ref{alg:always}). However, when the analysis focuses on specific dimensions (e.g., spatial coordinates excluding orientation), it is permissible and often practical to perform intersection checks between BRSs and relevant predicate slices. These intersecting slices should be incorporated into the subsequent computation of BRSs for the relevant disturbance set.
\end{remark}



\subsection{Analysis}
We wish to ascertain that the procedure described in Algorithm.~\ref{alg:cap} is sound, which we state as follows.
\begin{theorem}
    Using Algorithm.~\ref{alg:cap}, the resulting set of hybrid controllers $K^*$ is disturbance-robust to any strategy from $\mathcal{D}^*$ while satisfying specification $\phi$.
\end{theorem}
\begin{proof}
    First, note that Alg.~\ref{alg:cap} returns $\mathcal{D}^*$ and hybrid controller $K^*$ from the optimal subformula $\Psi_j \in \phi$. This takes care of any disjunction in $\phi$.
    For the conjunction, we consider the concatenation of the different temporal operators, given Assumptions~\ref{ass:predicates} and~\ref{ass:intervals}, as follows.
    
    For $\Box_{I_j}p_j$ preceded by $\Box_{I_{j+1}}p_{j+1}$, the inclusion of the \emph{obstacle} $\neg p_j, \forall t \in I_j$ and the BRS computation until $\bar{I}_{j+1}$ for disturbance-level $\mathcal{D}^*, s.t.\quad \mathcal{G}_j(\bar{I}_{j+1}) \supseteq \mathcal{X}_{p_{j+1}}$ implies that $\forall x(\bar{I}_{j+1}) \in \mathcal{X}_{p_{j+1}}, x(t') \in \mathcal{X}_{p_j}, \forall t' \in I_j, \forall d(\cdot) \in \mathcal{D}_{\psi_{j}} \supseteq \mathcal{D}^*$. 
    For $\Box_{I_j}p_j$ preceded by $(\diamondsuit \lor \mathcal{U})$, we take this same approach but collect all BRSs, $K$s and $\mathcal{D}$'s for which $\exists x(t') \in \mathcal{X}_{p_{j+1}}, t' \in I_{j+1}$ which ensures that $x(t'') \in \mathcal{X}_{p_j}, \forall t'' \in I_j, \forall d(\cdot) \in \mathcal{D}_{\psi_{j}} \supseteq \mathcal{D}^*$ with Eq.~\eqref{eq:resolve}. 
    For $\diamondsuit_{I_j}p_j$ preceded by $\Box_{I_{j+1}}p_{j+1}$, note that again, we collect all BRSs $s.t.\hquad x(\bar{I}_{j+1}) \in \mathcal{X}_{p_{j+1}} \implies \exists t\in I_j, s.t.\hquad x(t) \in \mathcal{X}_{p_j}, \forall d(\cdot)\in \mathcal{D}_{\psi_{j}} \supseteq \mathcal{D}^*$ with help of Lemma.~\ref{lem:shifting}.
    For $\diamondsuit_{I_j}p_j$ preceded by $(\diamondsuit \lor \mathcal{U})$, we collect all BRSs, $K$s and $\mathcal{D}$s $s.t.\hquad \exists t' \in I_{j+1}, s.t.\hquad x(t') \in \mathcal{X}_{p_{j+1}} \implies \exists t'' \in I_{j}, s.t.\hquad x(t'') \in p_{j}$. Eq.~\eqref{eq:resolve} then ensures satisfaction under  $\mathcal{D}_{\psi_{j}} \supseteq \mathcal{D}^*$.
    For $p_{j,1} \mathcal{U}_{I_j} p_{j,2}$ preceded by $\Box_{I_{j+1}}p_{j+1}$, note that we obtain $(K_1,K_2)$. $K_1$ ensures that, $\forall x(\bar{I}_{j+1}) \in \mathcal{X}_{p_{j+1}}, x(\underline{I}_j) \in \mathcal{X}_{p_{j,1}}$ and $\exists t'\geq \underline{I}_j s.t.\hquad x(t)\in \mathcal{X}_{p_{j,1}}, \forall t \in [\underline{I}_j,t'], \forall d(\cdot) \in  \mathcal{D}_{\psi_{j}} \supseteq \mathcal{D}^*$ with Eq.~\eqref{eq:resolve}. $K_2$ ensures that $\exists t' \in I_j, s.t.\hquad x(t) \in p_{j,1} \forall t \in [\underline{I}_j,t']$ and $\exists t'' \in [t',\bar{I}_j] s.t.\hquad x(t) \in p_{j,2} \forall t \in [t',t''], \forall d(\cdot) \in \mathcal{D}_{\psi_{j}} \supseteq \mathcal{D}^*$ with Eq.~\eqref{eq:resolve}.
    For $p_{j,1} \mathcal{U}_{I_j} p_{j,2}$ preceded by $(\diamondsuit \lor \mathcal{U})$, $K_1$ ensures that $\exists t \in I_{j+1}, s.t.\hquad x(t) \in \mathcal{X}_{p_{j+1}}, x(\underline{I}_j) \in \mathcal{X}_{p_j}$. The rest follows according to the previous case.
    Eq.~\eqref{eq:resolve} ensures that for any operator sequence with temporal flexibility, this is under the maximum possible disturbance set $\mathcal{D}_{\psi_j} \supseteq \mathcal{D}_{\Psi_i}$ while temporal flexibility is resolved for any $t' \in I_j$.
    By construction of Alg.~\ref{alg:cap}, and specifically (\emph{line 13}) and Eq.~\eqref{eq:resolve}, we ensure that $\mathcal{D}^* \subseteq \mathcal{D}_{\Psi_i} \subseteq \mathcal{D}_{\psi_{j}}, \forall j \in \{1,...,N_i\}$ and hence subformula $\Psi_i$ can be satisfied $\forall x(0) \in \mathcal{X}_0$ under hybrid control law $K^*(t,x)$.

\end{proof}


\section{Results}
\label{sec:results}
We will show how Alg.~\ref{alg:cap} solves the maximal disturbance-robust STL control synthesis and simultaneously obtains an under-approximation of the maximal permissible disturbance-set, $\mathcal{D}^*$, and the hybrid controller $K^*(x,t)$.

All simulations are performed on an Intel Core 12700H processor with 32 Gb of RAM. The method is implemented in Python using the OptimizedDP package~\cite{bui2022optimizeddp}.

In our simulations, we consider a planar Autonomous Underwater Vehicle (AUV). The system might be subjected to considerable underwater currents that have not been assessed before. We consider the frontal and side surface area as a scaling of the disturbance on the $x$ and $y$ coordinates. Hence, consider the nonlinear dynamical system
\begin{equation*}
    \begin{bmatrix} \dot{x} \\ \dot{y} \\ \dot{v} \\ \dot{\theta} \end{bmatrix} = \begin{bmatrix} v \cos(\theta) + d_x(A_{front} \cos(\theta) + A_{side} \sin(\theta)) \\ v \sin(\theta) + d_y(A_{front}\sin(\theta) + A_{side}\cos(\theta)) \\ u_v \\ u_{\theta} \end{bmatrix},
\end{equation*}
where $d_x \subseteq \mathcal{D}_x \in \mathbb{R}$ and $d_y \subseteq \mathcal{D}_y \in \mathbb{R}$ are the disturbances in $x$ and $y$ dimension, $A_{front}=0.1$ and $A_{side}=1.0$ are the frontal and side surface area of the AUV respectively, $u_v \in [-4,4]$ is a control input to change the forward or backward velocity of the AUV, and $u_{\theta} \in [-1,1]$ is a control input to change the orientation of the AUV in the plane.

\subsection{Illustrative Example}
In Fig.~\ref{fig:intro}, an AUV is subjected to a vertical disturbance ($\mathcal{D}_x = \varnothing$) and specification $\phi = \Box_{[0,0]}(x\in \mathcal{X}_0) \land \Box_{[7,8]}(x \in A) \land \Box_{[15,15]}(x\in \mathcal{X}_f)$. The $\mathcal{D}_y$ disturbance and the underactuation of the AUV make it difficult to stay in a narrow corridor without leaving $A$ or colliding with the rocks. In contrast to space- and time-robust methods, our maximal disturbance-robust control synthesis ensures the desired behavior (being orthogonal to the rock due to the reduced area subject to the vertical disturbance) and quantifies the maximum disturbance that can be accommodated during $\phi$. 
We obtain $\mathcal{D}^* = [-1,1]$ with accuracy $\epsilon = 0.05$.

\begin{figure}[t!]
    \centering
    \includegraphics[width=0.45\textwidth]{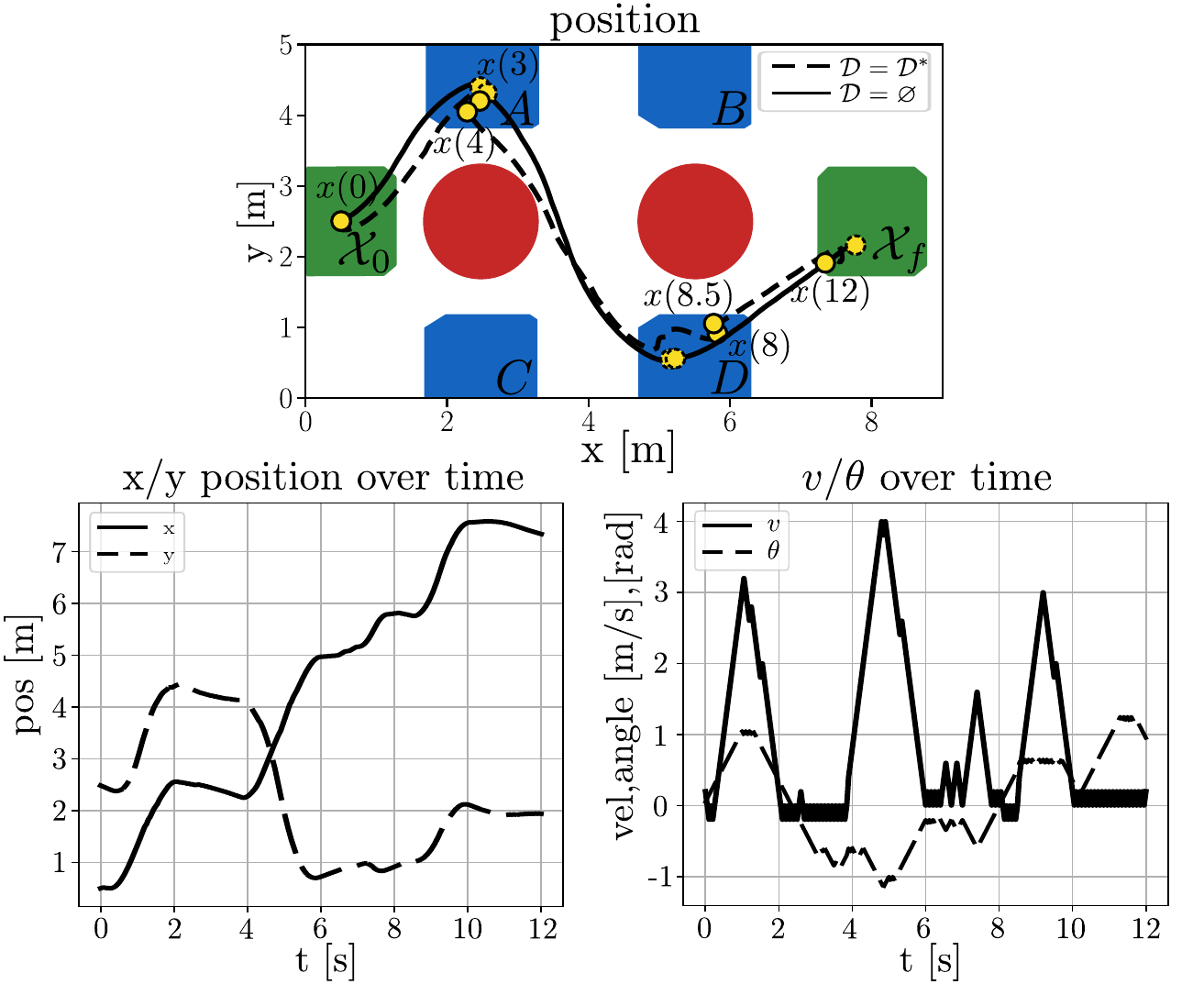}
    \caption{An AUV subjected to an STL specification after finding the maximum permissible disturbance. The system is guaranteed to satisfy $\phi$ $\forall \mathcal{D} \subseteq \mathcal{D}^*$. We show the maximally disturbed and undisturbed case.}
    \label{fig:results}
    \vspace{-0.5cm}
\end{figure}

\subsection{AUV Example}
\label{ssec:auv}
We consider a more complex example where the goal is to find the controller and maximal disturbance sets $\mathcal{D}^*_x$ and $\mathcal{D}^*_y$ subjected to the inspection specification
\begin{equation*}
\begin{aligned}
    \phi = &(\diamondsuit_{[3,4]}(x\in A) \land \Box_{[8,8.5]}(x \in D)) \lor \\ &(\Box_{[4,4.5]}(x\in C) \land \diamondsuit_{[8,9]}(x\in B)),
\end{aligned}
\end{equation*}
with initial state $\Box_{[0,0]} (x(0) \in \mathcal{X}_0)$ and final state $\Box_{[12,12]}(x(t_f) \in \mathcal{X}_f)$. We define initial and final conditions as instantaneous \emph{Always} operators to be consistent with using Alg.~\ref{alg:cap}. The specification consists of two subformulas. 
\\\\
We obtain a maximal disturbance robustness degree of $\hat\delta_{\phi} = 0.3: \mathcal{D}^*_x = [-0.3,0.3], \mathcal{D}^*_y = [-0.3,0.3]$ with accuracy $\epsilon = 0.05 \frac{m}{s}$. Results are shown in Fig.~\ref{fig:results} which show the undisturbed trajectory (under $\mathcal{D} = \varnothing$) and the worst-case trajectory (under $\mathcal{D} = \mathcal{D}^*$) using the optimal control from Eq.~\eqref{eq:opt_control}.
The limiting factor in the maximally robust control synthesis is $\Box_{[8,8.5]}(x\in D)$ which requires a stay-in objective under significant side-ward disturbances (which cannot be directly counteracted). 
The temporal flexibility of the preceding operator ($\diamondsuit_{[3,4]}(x\in A)$) makes this realizable for $\hat\delta_{\Box_{[8,8.5]}} = 0.05$ (resolving at $t'=5$) and $\hat\delta_{\Box_{[8,8.5]}} = 0.3$ (resolving at $t'=4$), where Eq.~\eqref{eq:resolve} ensures the maximal disturbance set is obtained. 
In contrast, the maximal disturbance-robust degree to go from $\mathcal{X}_f$ to $D$ in 3.5 seconds is $\hat\delta_{\Box_{[12,12]}} = 0.5$. 
The second subformula has the maximal disturbance-robust degree of $\hat\delta_{\Psi_2} = 0.15$. 
Alg.~\ref{alg:cap} ensures termination of this subformula upon discovering a lower maximum. 
With $\mathcal{D}_0 = \{\mathcal{D}_{x,0},\mathcal{D}_{y,0}\} = \{[-0.65,0.65],[-0.65,0.65]\}$ our computation takes 93 seconds. 
\\\\
For all scenarios, the computation time is heavily dependent on $\epsilon$, $\mathcal{D}_0$, the grid sizes, and the ordering of the subformulae.


\section{Conclusions}
\label{sec:conclusions}
We have extended upon semantics of Signal Temporal Logic that quantify the disturbance-robust capabilities of a dynamical system subjected to an STL specification. 
Additionally, we presented an algorithm that synthesizes the hybrid controller that satisfies the STL specification under the maximal possible disturbance set. 
We have shown the soundness of our approach and shown its capabilities in a simulation study.

Future work will involve expanding the fragment of STL while ensuring computational tractability and using learning-based heuristics and sound under approximations for faster approximate solutions.

\bibliographystyle{IEEEtran}
\bibliography{references}

\end{document}